\documentclass[11pt]{article}
\usepackage[margin=3cm]{geometry}
\RequirePackage[colorlinks,citecolor=blue,urlcolor=blue]{hyperref}

\usepackage[shortlabels, inline]{enumitem}
\usepackage[small]{caption}

\usepackage{url}            %
\usepackage{booktabs}       %
\usepackage{amsfonts}       %
\usepackage{nicefrac}       %
\usepackage{microtype}      %
\usepackage{xcolor}         %

\usepackage{graphicx}
\usepackage{setspace}
\usepackage{apptools}
\usepackage{microtype}

\usepackage{authblk}

\usepackage{amsthm}
\usepackage{amsmath}
\usepackage{amssymb}

\usepackage{bm}

\usepackage{graphicx}
\usepackage{subcaption}
\graphicspath{{./figures/}}

\newcommand\titlestr{Target alignment in truncated kernel ridge regression}

\title{\titlestr}

\author{Arash A. Amini\thanks{aaamini@ucla.edu}}
\affil{Department of Statistics, UCLA}
\author{Richard Baumgartner\thanks{richard\_baumgartner@merck.com}}
\affil{Merck \& Co., Inc. Rahway, New Jersey, USA}
\author{Dai Feng\thanks{dai.feng@abbvie.com}}
\affil{Data and Statistical Sciences AbbVie Inc. North Chicago, IL, USA}

\usepackage{arash_macros}

\newcommand{\fh}{\widehat{f}}
\newcommand{\fs}{f^*}

\newcommand{\ft}{\widetilde f}
\newcommand{\hil}{\mathbb H}
\newcommand{\hilt}{\widetilde \hil}
\newcommand{\omh}{\widehat \omega}
\newcommand{\omt}{\widetilde \omega}

\newcommand{\Ker}{\mathbb K}
\newcommand{\Kert}{\mathbb{\widetilde K}}

\newcommand{\yt}{\widetilde{y}}

\DeclareMathOperator{\mse}{MSE}

\newcommand{\Xc}{\mathcal X}

\newcommand{\hip}[1]{\ip{#1}_\hil}
\newcommand{\empnorm}[1]{\norm{#1}_n}
\newcommand{\empip}[1]{\ip{#1}_n}

	\newcommand{\wt}{\widetilde{w}}

\newcommand{\Kt}{\widetilde K}

\newcommand{\Gaml}{\Gamma_\lambda}

\DeclareMathOperator{\cov}{cov}
\newcommand\hilnorm[1]{\norm{#1}_\hil}
\newcommand\hiltnorm[1]{\norm{#1}_{\hilt}}
\newcommand\hiltip[1]{\ip{#1}_{\hilt}}
\newcommand\ind[1]{1{\{#1\}}}

\newcommand\xv{\bm x}
\newcommand\St{\widetilde S}
\newcommand\xis{\xi^*}

\newcommand\mseb{\overline{\mse}}

\newcommand{\ltwoip}[1]{\ip{#1}_{L^2}}
\newcommand{\ltwonorm}[1]{\norm{#1}_{L^2}}

\begin{document}

\maketitle

\begin{abstract}
\label{abs}
Kernel ridge regression (KRR) has recently attracted renewed interest due to its potential for explaining the transient effects, such as double descent, that emerge during neural network training. In this work, we study how the alignment between the target function and the kernel affects the performance of the KRR. We focus on the truncated KRR (TKRR) which utilizes an additional parameter that controls the spectral truncation of the kernel matrix. We show that for polynomial alignment, there is an \emph{over-aligned} regime, in which TKRR can achieve a faster rate than what is achievable by full KRR. The rate of TKRR can improve all the way to the parametric rate, while that of full KRR is capped at a sub-optimal value. This shows that target alignemnt can be better leveraged by utilizing spectral truncation in kernel methods.
We also consider the bandlimited alignment setting and show that the regularization surface of TKRR can exhibit transient effects including multiple descent and non-monotonic behavior. Our results show that there is a strong and quantifable relation between the shape of the \emph{alignment spectrum} and the generalization performance of kernel methods, both in terms of rates and in finite samples.

\end{abstract}

\section{Introduction}
\label{Intro}
Kernel methods have become a time-proven popular mainstay in machine learning~\cite{steinwart2008support,wahba1990spline,scholkopf2001learning}. Implicit transformation of a learning problem via a suitable kernel with subsequent development of regularized linear models in the associated reproducing kernel Hilbert space (RKHS) is amenable to applications and allow for investigation of theoretical properties of the kernel methods \cite{cui2021generalization}.  Recently, there has been a revived interest and increased research into the kernel methods \cite{jacot2018neural,kobak2020theoptimal,bordelon2020spectrum, cui2021generalization,canatar2021spectral}. These efforts have been motivated by the connections between neural network and  kernel learning \cite{jacot2018neural}. In particular, the discovery of transient effects such as double descent that emerge in  neural networks training \cite{belkin2019reconciling,loog2020abrief} has attracted a lot of attention.    

In this paper, we study the effect of target-model alignment on the performance of kernel ridge regression (KRR) estimators. KRR is a well-known nonparametric regression estimator with desirable theoretical performance guarantees. We focus on a generalization of the KRR, which we refer to as the Truncated Kernel Ridge Regression (TKRR), that introduces an additional parameter, $r$, for truncating the eigenvalues of the underlying kernel matrix. TKRR is often used as an approximation to the full KRR to scale the computations to large sample sizes. TKRR itself is often approximated in practice by various sampling and sketching schemes~\cite{williams2001using,zhang2008improved,kumar2009ensemble,li2010making,talwalkar2014matrix,alaoui2015fast,yang2017randomized}. One of our main contributions in this paper is that TKRR can, in fact, improve the statistical rate of convergence in cases where the target function is well-aligned with the underlying RKHS. Roughly speaking, target alignment refers to the decay rate of the coefficients of the target function in an orthonormal basis of the eigenfunctions of the kernel operator.

In addition to improved rates for TKRR, we also show that target alignment generally predicts the generalization performance of the full KRR as well as TKRR. We also explore the regularization surface of the KRR estimators in both the regularization and truncation parameters, and show some surprising behavior like non-monotonicity, multiple descent and phase transition phenomena.

\textbf{Our main contributions}
\begin{enumerate}[(i)]\itemsep=0ex%
    \item We show that TKRR itself can be viewed as a full KRR estimator in a smaller RKHS embedded in the original RKHS (Propostion~\ref{prop:TKRR:equivalence}). This shows that the ad-hoc truncation that is often done for computational reasons can itself be viewed as a proper KRR estimator over a different RKHS.
    \item We introduce target alignment (TA) scores in Section~\ref{sec:target:align} and derive an exact expression for the Mean Sqaured Error (MSE) of TKRR based on the TA scores and eigenvalues of the kernel matrix (Theorem~\ref{thm:exact:mse}).
    \item For the case of polynomially-decaying TA scores, 
    we first derive a simplified expression for the MSE (Eqn.~\eqref{eq:poly:rate:mse}) which is within constant factors of the exact expression. We then identify four target alignment regimes: Under-aligned, just-aligned, weakly-aligned and over-aligned. We  show that TKRR can achieve 
	a fast rate %
	of convergence
	in the \emph{over-aligned} regime. In contrast, the rate of the full KRR is capped to a sub-optimal level, that of the best achievable in the weakly-aligned regime %
	(Theorem~\ref{thm:poly:rate}).
    
    \item For bandlimited TA scores, we provide more refined finite-sample analysis. In particular, we show that the regularization curve as a function of the truncation level $r$ is in general non-monotonic. We also show that TA spectra that are narrower and shifted towards lower indices correspond to lower generalization errors (Proposition~\ref{prop:band:lim}).
    
\end{enumerate}

We provide experiments verifying the multiple-descent and non-monotonic behavior of the regularization curves as well as the improved rate of Theorem~\ref{thm:poly:rate} (Section~\ref{subsec:polyAli}). The MSE in our results, is equivalent to the excess (empirical) generalization error. So all the conclusions, and in particular the rates, are valid for the generalization error. Overall, our results provide strong theoretical and experimental support for the idea that alignment of the target function with the eigenfunctions of the kernel is the key predictor of generalization performance of KRR estimators. It is worth noting that our main theoretical results are non-asymptotic and at times exact.

\textbf{Related work.}
 Kernel ridge regression (KRR) has been widely studied in the literature in the non-parametric statistical framework \cite{wasserman2006all,tsybakov2009intro}. Various aspects of the KRR were elucidated. For example, the error rates of the KRR were studied in \cite{caponnetto2007optimal} and more recently benign overfitting \cite{bartlett2020benign} and related double descent effects in \cite{hastie2022surprises}.
 
 The notion of kernel-target alignment was first introduced in \cite{cristianini2001onkernel} for classification. In \cite{canatar2021spectral}, the generalization error of the KRR was investigated using the tools of statistical physics theory which allows for derivation of the bias-variance decomposition of the generalization error and elucidation of non-monotonic transitional effects reflected by the learning curves.  In this work the authors have also demonstrated that the interplay between spectral bias (expressed as kernel eigenvalue decay) and kernel-target alignment or task-model alignment are useful for the characterization of the kernel compatibility with the problem.  Building on the \cite{canatar2021spectral}, the concept of kernel-target alignment and spectral bias was referred to as source and capacity, respectively in \cite{cui2021generalization}. Similarly to \cite{canatar2021spectral}, the transitions that emerge during a crossover from noiseless to noisy regime corresponding to fast vs slow decay of the learning curve, respectively were studied in \cite{cui2021generalization}. Moreover, in  \cite{amini2021tkrr} it has been shown that  the truncated KRR outperforms full KRR over the unit ball of the RKHS. The work~\cite{cui2021generalization} is perhaps closest to ours followed by~\cite{amini2021tkrr}. We make more detailed comparisons with these after Theorem~\ref{thm:poly:rate}. In particular, we show that the rate obtained in~\cite{cui2021generalization} can be improved by using TKRR. In addition, in Appendix~\ref{app:rkhs:background}, we provide a more unified view of the settings of \cite{caponnetto2007optimal, cui2021generalization} which can help clarify how different target alignment assumptions relate to each other. It is worth noting that approximate forms of TKRR via sampling and sketching have been analyzed extensively in the literature~\cite{rudi2015less,yang2017randomized, cortes2010impact,yang2012nystrom,jin2013improved,alaoui2015fast,bach2013sharp}, but this line of work at best proves that TKRR achieves the same minimax rate achievable with full KRR, over the RKHS ball. What we show in this paper is that TKRR can achieve much faster rates under proper target alignment. %

\section{Kernel Ridge Regression}
\label{sec:KRR}
The general nonparametric regression problem %
can be stated as 
\begin{align}\label{eq:mod:1}
	y_i = \fs(x_i) + w_i, \;i = 1,\dots,n, \quad \ex w = 0,\; \cov(w) = \sigma^2I_n
\end{align}
where $w = (w_i) \in \reals^n$ is a noise vector and $\fs: \Xc \to \reals$ is the function of interest to be approximated from the noisy observations $y = (y_i)$. Here, $\Xc$ is the space to which the \emph{covariates} $\bm x = (x_i)$ belong. We mainly focus on the fixed design regression where the covariates are assumed to be deterministic. A natural estimator is the kernel ridge regression (KRR), defined as the solution of the following optimization problem:
\begin{align}\label{eq:krr:f:1}
	\fh_{n,\lambda} := \argmin_{f \in \hil} 
	\;\frac1{n} \sum_{i=1}^n (y_i - f(x_i))^2 + \lambda \norm{f}_\hil^2,
\end{align}
where $\lambda > 0$ is a regularization parameter and $\hil$ is the underlying RKHS.  By the representer theorem~\cite{kimeldorf1971some}, this problem can be reduced to a finite-dimensional one:
\begin{align}\label{eq:krr:omg:1}
	\min_{\omega \,\in\, \reals^n} \; 
	\frac1{n} \vnorm{y- \sqrt n K \omega}^2 + \lambda
	\omega^T K \omega, \quad \text{where}\quad  K =\frac1n \big( \Ker(x_i,x_j)\big) \in \reals^{n \times n}
\end{align} 
is the (normalized empirical) kernel matrix. 
Let us define the following operators
\begin{align}
	S_{\xv}(f) = \frac1{\sqrt n} (f(x_1),f(x_2),\dots,f(x_n)), \quad 
	S_{\xv}^*( \omega ) = \frac{1}{\sqrt n} \sum_{j=1}^n \omega_j \Ker(\cdot, x_j)
\end{align}
for $f \in \hil$ and $\omega \in \reals^n$. We refer to $S_{\xv}$ as the \emph{sampling operator}.  As the notation suggests, $S_{\xv}^*$ is the adjoint of $S_{\xv}$ as an operator from $\hil$ to $\reals^n$. We have $\fh_{n,\lambda} = S_{\xv}^*( \omh) $ where $\omh$ is the solution of~\eqref{eq:krr:omg:1}. We note that $S_{\xv} (S^*_{\xv} (\omega)) = K \omega$, i.e., $K$ is the matrix representation of $S_{\xv} S^*_{\xv} : \reals^n \to \reals^n$.

\textbf{Notation.} We write $\empip{f,g} = \frac1n \sum_{i=1}^n f(x_i)g(x_i)$ and $\empnorm{f} = \sqrt{\empip{f,f}}$ for the empirical inner product and norm, respectively. Note that $\empnorm{f} = \norm{S_{\xv}(f)}_2$ where $\norm{\cdot}_2$ is the vector $\ell_2$ norm. For two functions, $f$ and $g$, we write $f \lesssim g$ if there is constant $C > 0$ such that $f \le C g$, and we write $f \asymp g$ if $f \lesssim g$ and $g \lesssim f$.

\section{Truncated Kernel Ridge Regression}\label{sec:tKRR}
As mentioned earlier, spectral truncation of the kernel matrix has been mainly considered as a computational device in approximating the full KRR estimator. Here, we first argue that it can be reformulated as a KRR estimator over a simpler RKHS. To facilitate future developments, we define this new RKHS based on the given training data $\xv = (x_1,\dots,x_n)$.

Let $K = \sum_{k=1}^n \mu_k u_k u_k^T$ be the eigendecomposition of the kernel matrix $K$ where $\mu_1 \ge \mu_2 \ge \dots \ge \mu_n \ge 0$ are the eigenvalues of $K$ and $\{u_k\} \subset \reals^n$ are the corresponding eigenvectors. We assume for simplicity that $\mu_n> 0$, that is:
\begin{assum}\label{assum:inv}
	The (exact) kernel matrix $K = K(\xv)$ is invertible.
\end{assum}
This is very mild assumption for any infinite-dimensional $\hil$.
 Let $u_{ki}$ be the $i$th coordinate of $u_k$. We would like to find functions $\{\psi_k\} \subset \hil$ such that $\psi_k(x_i) = \sqrt n\, u_{ki}$ for all $k, i \in [n]$. If $\dim(\hil) = \infty$, there are in general many such functions for each $k$. We pick the one that minimizes the $\hil$-norm: %
 \begin{align*}
 	\psi_k := \argmin\ \bigl\{ \hilnorm{\psi_k} : \psi \in \hil, \; S_{\xv}(\psi) = \sqrt n\, u_k\bigr\}.
 \end{align*}
 Under Assumption~\ref{assum:inv}, the above problem has a unique solution (i.e., the interpolation problem $S_{\xv}(\psi) = \sqrt n u$ has a unique minimum-norm solution for any $u \in \reals^n$.). Since $\{u_k\}$ is an orthonormal basis of $\reals^n$, by construction, $\empip{\psi_k, \psi_\ell} = \ind{k = \ell}$. Consider the set of functions %
 \begin{align*}
 	\hilt := \Bigl\{ \sum_{k=1}^r \alpha_k \psi_k \mid   \alpha_1,\dots,\alpha_r \in \reals \Bigr\}
 \end{align*}
	equipped with the inner product defined via $\hiltip{\psi_k, \psi_\ell} =  \frac1{\mu_k} \ind{k = \ell}$ and extended to the whole of $\hilt$ by bilinearity. (Equivalently, if $f = \sum_{k=1}^r \alpha_k \psi_k$ and $g = \sum_{k=1}^r \beta_k \psi_k$, define $\hiltip{f,g} = \sum_k\alpha_k \beta_k/\mu_k$.) Then, $\hilt$ is an $r$-dimensional RKHS, with reproducing kernel (Appendix~\ref{app:verify:reproduce})
	\begin{align}\label{eq:Kt:def}
		\Kert(x,y) := \sum_{k=1}^r \mu_k \psi_k(x) \psi_k(y)
	\end{align}
	 and by construction $\hilt \subset \hil$, although $\hilt$ and $\hil$ are equipped with different inner products. 

\subsection{TKRR is itself a valid KRR}
\label{tKRRestimator}
We define the truncated kernel ridge regression (TKRR) estimator as the usual KRR estimator over $\hilt$, that is, any solution of the following problem
\begin{align}\label{eq:TKRR}
	\ft_{r,\lambda} \in \argmin_{f \in \hilt} 
	\;\frac1{n} \sum_{i=1}^n (y_i - f(x_i))^2 + \lambda \hiltnorm{f}^2.
\end{align}
The solution to this problem is in general not unique. However, all the solutions will produce the same values at points $x_1,\dots,x_n$ as the following result shows:
\begin{proposition}\label{prop:TKRR:equivalence}
 Let $\Kt = \sum_{k=1}^r \mu_k  u_k u_k^T = (\frac1n \Kert(x_i,x_j)) $ be the truncated kernel matrix and 
	\begin{align}
		\St_{\xv}^*(\omega) := \frac1{\sqrt n} \sum_{j=1}^n \omega_j \Kert(\cdot, x_j),
	\end{align}
which is the adjoint of $S_{\xv}$ as a map from $\hilt$ to $\reals^n$. 	Let $\widetilde\Omega$ be the solution set of~\eqref{eq:krr:omg:1} with $K$ replaced with $\Kt$. 
	Then, the following hold:
	\begin{enumerate}[(a), wide]\itemsep=0pt
		\item 	All the solutions $\ft_{r,\lambda}$ of~\eqref{eq:TKRR} are mapped to the same point in $\reals^n$ by the sampling operator $S_{\xv}$.
		\item %
		$\{ \St_{\xv}^*(\omega) \mid \omega \in \widetilde\Omega\}$ is the solution set of~\eqref{eq:TKRR}.
		\item For $r = n$, there is a unique solution $\ft_{r,\lambda} = \fh_{n, \lambda}$.
	\end{enumerate}
\end{proposition}
A consequence of the Proposition~\ref{prop:TKRR:equivalence} is that $\empnorm{\ft_{r,\lambda} - \fs}$ is the same no matter what TKRR solution we use. Although the above proposition is not the main contribution of this work, it provides the interesting observation that one can view TKRR as an exact KRR estimator over a smaller RKHS.

\section{Target alignment}\label{sec:target:align}

We start with the definition of target alignment:
\begin{definition}\label{defn:TA} The (target) alignment spectrum of a target function $\fs$ is 
    $\xi^* = U^T S_{\xv} (\fs) \in \reals^n$ where $U$ is the orthogonal matrix of the eigenvectors of $K$, with columns corresponding to eigenvalues of $K$ ordered in a decreasing fashion. The elements of $\xi^*$ are referred to as target alignment (TA) scores. 
\end{definition}

The following theorem gives an exact expression for the expected empirical MSE of TKRR estimate, in terms of the TA scores $\xis$ and the eigenvalues $(\mu_i)$ of the kernel matrix: 
\begin{theorem}[Exact MSE]\label{thm:exact:mse}
    Let $\Gaml$ be an $n\times n$ diagonal matrix with diagonal elements:
    \begin{align}\label{eq:Gaml:def}
    	(\Gaml)_{ii} = \dfrac{\mu_i}{\mu_i + \lambda} 1\{1 \le i \le r\}.
    \end{align}
	For any TKRR solution $\ft_{r,\lambda}$, we have
	\begin{align}
		 \ex \empnorm{\ft_{r,\lambda} - \fs}^2 &= 
		\norm{(I_n - \Gaml) \xi^*}_2^2 +   \frac{\sigma^2}n \tr(\Gaml^2) \notag \\
		&=
    \sum_{i=1}^r \frac{\lambda^2}{(\mu_i + \lambda)^2} (\xis_i)^2 +
    \sum_{i=r+1}^n (\xis_i)^2
    + 
    \frac{\sigma^2}{n} \sum_{i=1}^r \frac{\mu_i^2}{(\mu_i + \lambda)^2} \label{eq:mse:expr} \\
    &= \empnorm{\fs}^2 + \sum_{i=1}^r \frac1{(\mu_i + \lambda)^2} \Bigl[ - a_i(\lambda)  (\xis_i)^2
    + \frac{\sigma^2}{n} \mu_i^2 \Bigr]\label{eq:mse:expr:2}
	\end{align}
where $a_i(\lambda) = (\mu_i+\lambda)^2 - \lambda^2$ and the expectation is w.r.t. the randomness in the noise vector $w$. 
\end{theorem}

Assume that the target function is normalized so that $\empnorm{\fs} = 1$. Since $U$ is an orthogonal matrix, this is equivalent to $\norm{\xis}_2 = 1$. Since the coefficient of $(\xis_i)^2$ in \eqref{eq:mse:expr:2} is negative, the more the alignment spectrum peaks near the lower indices, the smaller the expected MSE. This observation is made more precise in Proposition~\ref{prop:band:lim} below.

\subsection{Bandlimited alignment}
To see the implications of Theorem~\ref{thm:exact:mse}, let us first consider a \emph{bandlimited model} as follows: $\fs$ is randomly generated with alignment scores $\xis_i$ satisfying 
\begin{align}\label{eq:BL:model:xi:cond}
	\ex(\xis_i)^2 = \frac1{b} 1\{ \ell+1 \le i \le \ell + b\}.
\end{align}
for $\ell \ge 0$ and $b \ge 1$ (both integers). Then, we write $\fs \sim \mathcal B_{b,\ell}$ and note that $\ex \empnorm{f^*}^2 = 1$ for such $\fs$.  We write $\mseb := \ex \empnorm{\ft_{r,\lambda} - \fs}^2$ where the expectation is over the randomness in both $\fs$ and $w$---the noise vector in~\eqref{eq:mod:1}. One can think of this setting as a Bayesian model for $\fs$ and of $\mseb$ as the Bayes risk w.r.t. to aformentioned prior on $f^*$. We leave the specification of the prior open subject to the condition~\eqref{eq:BL:model:xi:cond} since this is the only assumption we need.
\begin{proposition}
 \label{prop:band:lim}
	With $\fs \sim \mathcal B_{b,\ell}$, we have
	\begin{align*}
		\mseb = 1 - \frac1b \sum_{i=\ell+1}^{(\ell+b) \wedge r} \frac{a_i(\lambda)}{(\mu_i + \lambda)^2}  
		   +  \frac{\sigma^2}{n} \sum_{i=1}^r \frac{\mu_i^2}{(\mu_i + \lambda)^2},
	\end{align*}
	Assume for simplicity that $\{\mu_i\}$ are distinct. Then, the following hold: 
	\begin{enumerate}[(a)]
		\item For a fixed $\ell$, $b$ and $\lambda$, the $\mseb$ as a function of $r$ is increasing in $[1,\ell]$, and increasing in $[\ell+b+1,n]$. Let $j^* = \min\{i \in [\ell+1,\ell+b] \mid 1 + \frac{2\lambda}{\mu_i} > \frac{\sigma^2}n b \}$. Then,  $\mseb$ as a function of $r$ is nondecreasing in $[\ell+1, j^*]$ and decreasing in $[j^*, \ell+b]$. \label{band:lim:r:curve}
		
		\item For fixed $r$, $b$ and $\lambda$, the $\mseb$ as a function of $\ell$ is increasing in $\ell \in [0, r-b]$. 
		
		\item For fixed $r$, $\ell$ and $\lambda$, the $\mseb$ as a function of $b$ is increasing in $b \in [1, r-\ell]$. 
		
	\end{enumerate}
\end{proposition}

Part (a) of Proposition~\ref{prop:band:lim} shows that if $ (\sigma^2/n) b < 1 + 2\lambda/ \mu_{\ell+1}$, then $r \mapsto \mseb$ is decreasing over the entire interval $[\ell+1,\ell+1+b]$. This result shows that $\mseb$ as a function of $r$ can be non-monotonic (go up, down and back up).  Similar observations as Proposition~\ref{prop:band:lim}(a) can be made about bandlimited signals consisting of multiple nonzero bands, e.g., $\fs = \frac1{\sqrt K}(\fs_1 + \cdots + \fs_K)$ where $\fs_i \in \sim \mathcal B_{b_i, \ell_i}$ with non-overlapping bands. Part~(b)  shows that alignment spectra that are concentrated near lower indices are better. Part~(c) shows that concentrated alignment spectra are better than diffuse ones.

\subsection{Polynomial alignment}\label{subsec:polyAli}
We now consider the case where the TA scores decay polynomially with a rate potentially faster or slower than what is required by merely belonging to the RKHS. In particular, assume that 
\begin{align}\label{eq:poly:decay:assump}
	\mu_i \asymp i^{-\alpha}, \quad (\xi_i^*)^2 \asymp i^{-2\gamma \alpha -1}
\end{align}
for some $\gamma > 0$ and $\alpha \ge 1$. 
Let us justify the choice of decay for $(\xi_i^*)^2$. Asymptotically, both $(\xi_i^*)^2$ and $\mu_i$ should decay similar to the population level TA scores and eigenvalues. %
Then, $f^* \in \hil$, if and only if $\sum_i (\xi_i^*)^2 / \mu_i \lesssim 1$ (see Appendix~\ref{app:rkhs:background}).
Assuming a polynomial decay for $(\xi_i^*)^2$, it follows that that $(\xi_i^*)^2$ should decay faster than $i^{-\alpha-1}$ for the sum to converge, that is $\lesssim i^{-\alpha-1-\delta}$ for some $\delta > 0$. Without loss of generality, we are assuming $\delta = (2\gamma-1) \alpha$. This parameterization is chosen to be consistent with the existing literature.  Thus, for a function $\fs \in \hil$, $\xis$ satisfies~\eqref{eq:poly:decay:assump} with $\gamma > 1/2$. The case $\gamma \le 1/2$ then corresponds to a target that does not belong to the RKHS.

The following result characterizes the performance of full KRR and TKRR under model~\eqref{eq:poly:decay:assump}:

\begin{theorem}\label{thm:poly:rate}
	Let $\eta = \min(r, \lambda^{-1/\alpha})$. 
	Under the polynomial decay assumption~\eqref{eq:poly:decay:assump},
	\begin{align}\label{eq:poly:rate:mse}
		 \ex \empnorm{\ft_{r,\lambda} - \fs}^2 
		\;\asymp \; \lambda^2 \max(1 , \eta^{-2(\gamma-1)\alpha}) + r^{-2\gamma \alpha} 1\{r < n\} + \frac{\sigma^2}{n} \eta.
	\end{align}
	\begin{enumerate}[(a),wide]
		\item \label{part:tkrr:rate}	Taking $\lambda \asymp (\sigma^2/n)^{\gamma \alpha/(2\gamma \alpha+1)}$ and $r \asymp (n/\sigma^2)^{1/(2\gamma\alpha+1)}$, TKRR achieves the following rate
		\begin{align}\label{eq:TKRR:optimal:rate}
			\ex \empnorm{\ft_{r,\lambda} - \fs}^2 
			\;\asymp\;  \Bigl(\frac{\sigma^2}{n}\Bigr)^{2\gamma \alpha/(2\gamma \alpha+1)} \quad \text{for $\gamma > 1$}.
		\end{align}		
		\item \label{part:full:rate}
		Assume $n^{-2\alpha} \lesssim \sigma^2 \lesssim n$, and let $\delta := \min(1,\gamma)$. Then, the best rate achievable by the full KRR is obtained for regularization choice $\lambda \asymp (\sigma^2/n)^{\alpha/(2\delta \alpha + 1)}$ and is given by
		\begin{align}\label{eq:full:KRR:optimal:rate}
		    \ex \empnorm{\ft_{r,\lambda} - \fs}^2 
			\;\asymp\;  \Bigl(\frac{\sigma^2}{n}\Bigr)^{2\delta \alpha/(2\delta \alpha+1)} \quad \text{for $\gamma > 0$}.
		\end{align}
	\end{enumerate}
\end{theorem}
Note that the exponent $\delta$ saturates to $\delta = 1$ for all $\gamma > 1$ giving the best rate $(\sigma^2/n)^{2\alpha/(2\alpha+1)}$ in those cases. 
In contrast, the rate in~\eqref{eq:TKRR:optimal:rate} achievable by TKRR improves for $\gamma > 1$ without bound, and in the limit of $\gamma \to \infty$, approaches $(\sigma^2/n)^{-1}$, the best parametric rate.

It is also well-known that the minimax rate over the unit ball of the RKHS with eigendecay $\mu_i \asymp i^{-\alpha}$ is $(\sigma^2/n)^{\alpha/(\alpha+1)}$. This can also be seen by letting $\gamma \to 1/2$ in either~\eqref{eq:TKRR:optimal:rate} or~\eqref{eq:full:KRR:optimal:rate}, removing any target alignment assumption beyond what is provided by the RKHS itself.

To summarize, let us define the rate exponent function,
\begin{align}\label{eq:rate:exponent}
    s(\gamma) := 2\gamma \alpha/(2\gamma \alpha+1).
\end{align}
There are four regimes of target alignment, implied by Theorem~\ref{thm:poly:rate}:
\begin{enumerate}
	\item Under-aligned regime, $\gamma \in (0,\frac12)$: The target is not in the RKHS. The best achievable rate is %
	$(\sigma^2/ n)^{s(\gamma)}$
	which is slower than the minimax rate over the ball of the RKHS, $(\sigma^2/n)^{s(\frac12)}$. %
 	
	\item Just-aligned regime, $\gamma = \frac12$ (to be precise $\gamma \downarrow \frac12$): The target is only assumed to be in the RKHS. The best rate achievable is the  minimax rate over the ball of the RKHS, $(\sigma^2/n)^{s(\frac12)}$. %
 	
	 \item Weakly-aligned regime, $\gamma \in (\frac12,1]$: The target is in the RKHS and more aligned with the kernel than what is implied by being in the RKHS. The best achievable rate is $(\sigma^2/ n)^{s(\gamma)}$
	 which is faster than the minimax rate over the ball of the RKHS, $(\sigma^2/n)^{s(\frac12)}$. %
	 The rate is achievable by the full KRR and hence TKRR.
 	
	 \item Over-aligned regime, $\gamma > 1$: The target is in the RKHS and is strongly aligned with the kernel. The best achievable rate is 
	 $(\sigma^2/ n)^{s(\gamma)}$
	 which is achieved by TKRR. The full KRR can only achieve the rate 
	 $(\sigma^2/ n)^{s(1)}$
	 in this case, which is the best achievable in the weakly-aligned regime. %
\end{enumerate}

Note that the best achievable rate in the first three regimes is attainable by full KRR (hence TKRR) while for the fourth regime, only TKRR can achieve the given rate.
It was shown in~\cite{caponnetto2007optimal} (cf. Appendix~\ref{app:rate:details} for details) that the rate given in the weakly-aligned regime is minimax optimal over the class of targets~\eqref{eq:poly:decay:assump} with $\gamma \in (1/2,1]$. We conjecture that the minimax optimality of this rate extends to the over-aligned regime, that is, we conjecture that TKRR is minimax optimal for all $\gamma > 1/2$.

The capped rate~\eqref{eq:full:KRR:optimal:rate} for the full KRR is the same as the one obtained in~\cite[Eqn.~(12)]{cui2021generalization}. The rate~\eqref{eq:TKRR:optimal:rate} for the TKRR in the over-aligned regime is new to the best of our knowledge.
The work of~\cite{amini2021tkrr} shows some improvement in TKRR over full KRR, but since the minimax rate was considered only over the unit ball, no difference in rates was shown between TKRR and full KRR. In contrast, by considering the smaller alignment class~\eqref{eq:poly:decay:assump}, we can show the rate differences. We also note that the cross-over effects noted in~\cite{cui2021generalization} for different regularization regimes follow easily from the simplified formula~\eqref{eq:poly:rate:mse}, as can be seen by inspecting the various cases in the proof of part~(b) of Theorem~\ref{thm:poly:rate}. To simplify the statement of Theorem~\ref{thm:poly:rate}, we have focused only on the optimal regularization regime.

\section{Simulations}
\label{sec:simulation}
We present various simulation results to demonstrate the multiple-descent and phase transition behavior of the regularization curves, and corroborate the theoretical results in Theorem \ref{thm:poly:rate}.
First, consider the  bandlimited alignment~\eqref{eq:BL:model:xi:cond} where the nonzero entries of $\xis$ are i.i.d. draws from $N(0,1)$ and $\xis$ is normalized to have unit norm. We use the Gaussian kernel $e^{-\norm{x-x'}/2h^2}$ in $d=4$ dimensions with bandwidth $h=\sqrt{d/2}$, applied to 200 samples generated from a uniform distribution on $[0,1]^d$, and let $r=l+b$.

The plots in Figure \ref{fig:1a} show the  multiple-descent and phase transition behavior of the $\lambda$-regularization curves (that is, expected MSE versus $-\log(\lambda)$)
for different values of the noise level $\sigma$. The corresponding contour plot is shown in Figure \ref{fig:1b}. The plots show a transition from monotonically decreasing at $\sigma = 0$ to monotonically increasing for large $\sigma$, with non-monotonic behavior for moderate values of $\sigma$.

\begin{figure}[t]
\centering
\begin{subfigure}{0.49\textwidth}
 \centering
 \includegraphics[scale=0.29]{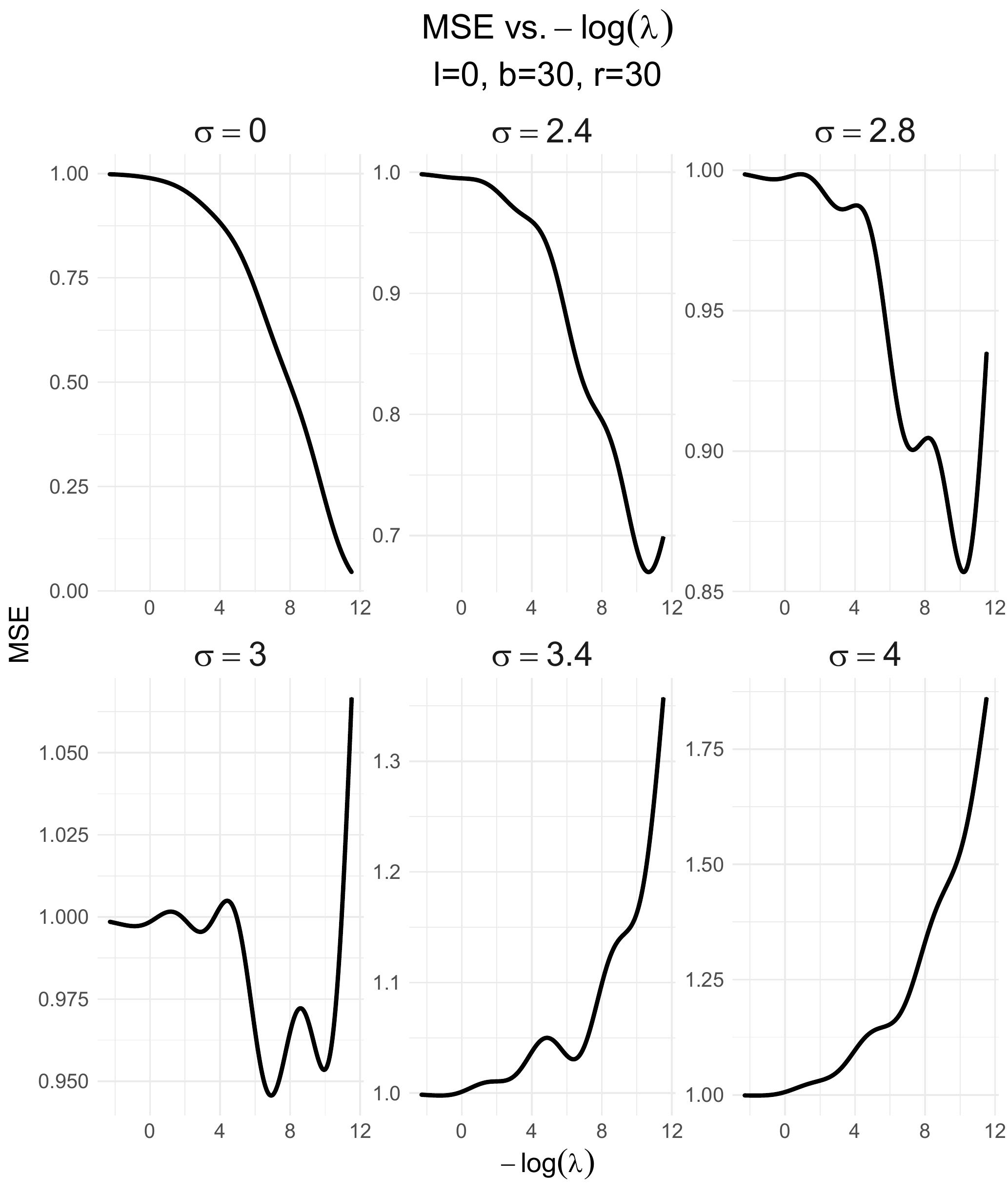}   
    \subcaption{Panel plot}
    \label{fig:1a}
 \end{subfigure}
\begin{subfigure}{0.49\textwidth}
 \centering
 \includegraphics[scale=0.57]{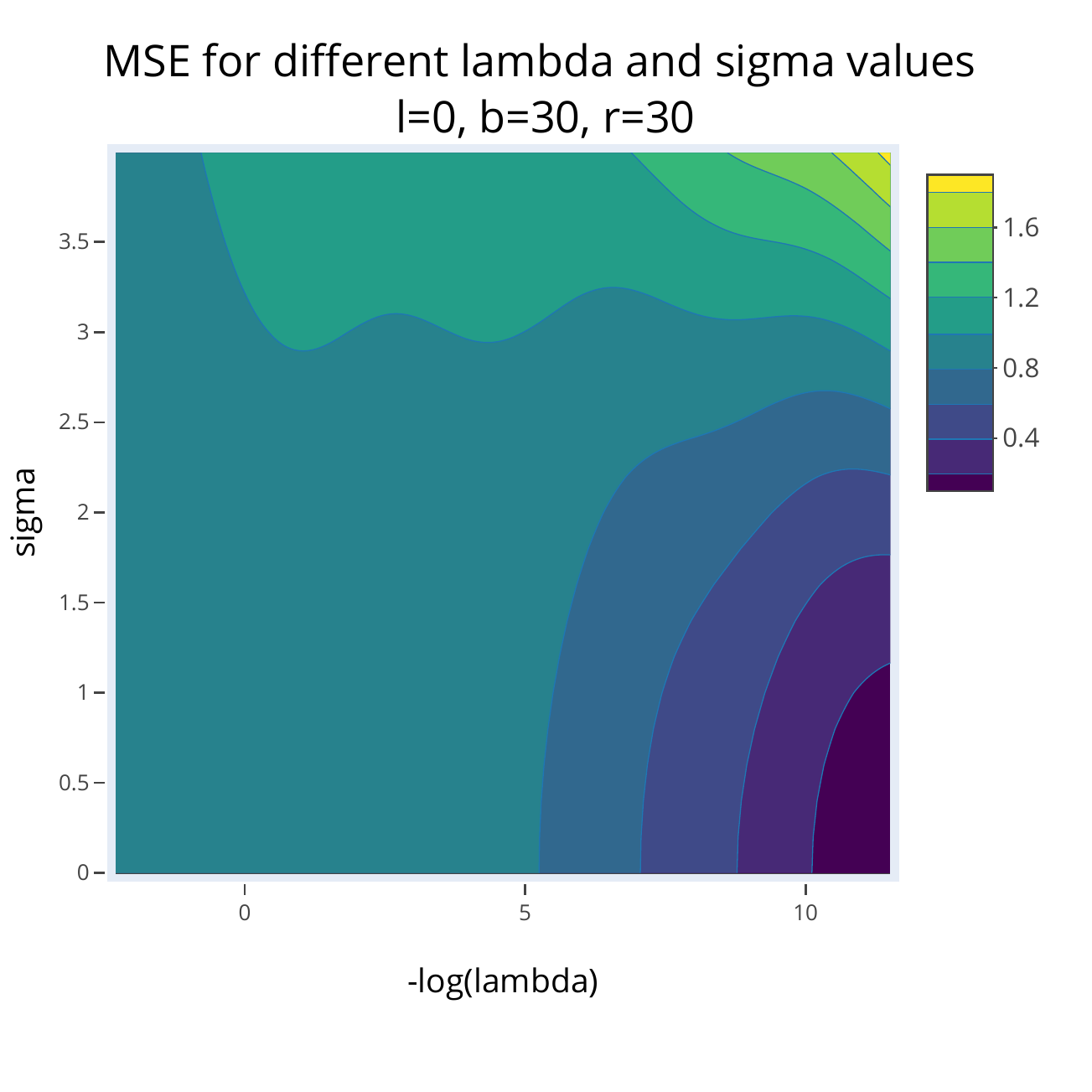}   
    \subcaption{Contour plot}
    \label{fig:1b}
 \end{subfigure}
 \caption{Multiple-descent and phase transition of 
 for $\lambda$-regularization curve:
 (a) Expected MSE as a function of $-\log(\lambda)$ %
 for different values of $\sigma$, 
 and (b) overall contour plot of expected MSE for $\sigma$ vs. $-\log(\lambda)$.}
\end{figure}

\begin{figure}[t]
\centering
\begin{subfigure}{0.49\textwidth}
 \centering
 \includegraphics[scale=0.29]{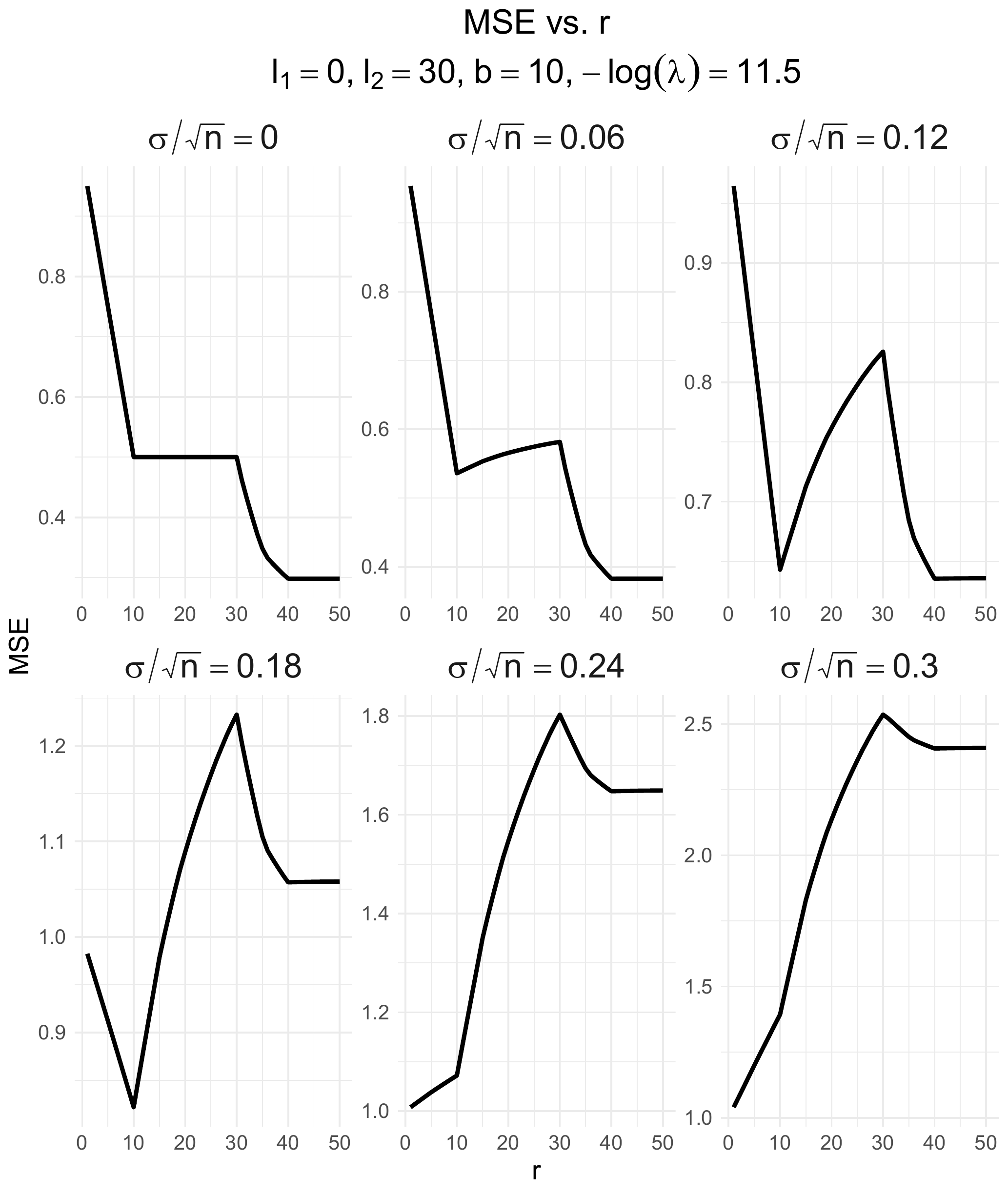}   
    \subcaption{Panel plot}
    \label{fig:2a}
 \end{subfigure}
\begin{subfigure}{0.49\textwidth}
 \centering
 \includegraphics[scale=0.55]{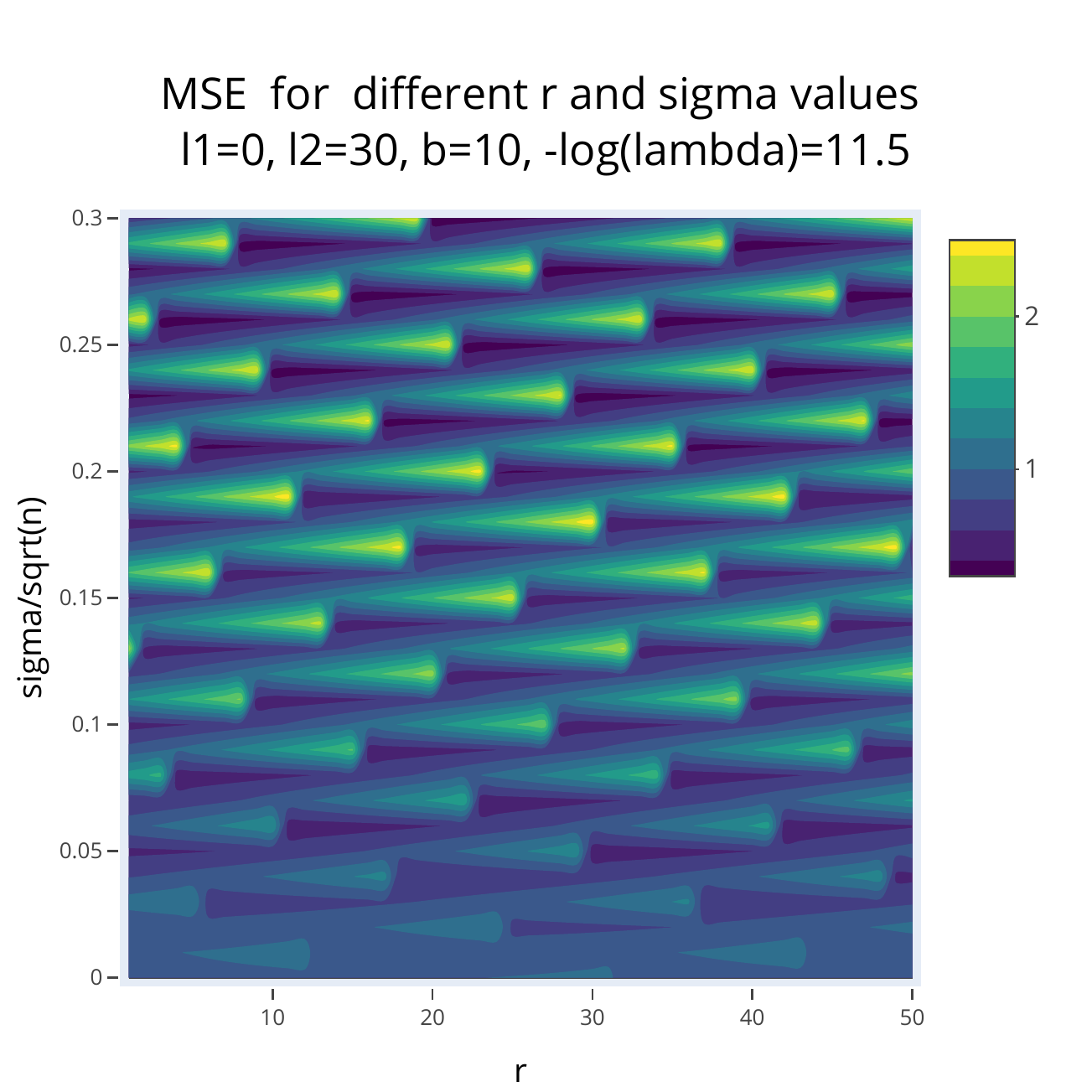}   
    \subcaption{Contour plot}
    \label{fig:2b}
 \end{subfigure}
  \caption{Double-descent and phase transition for $r$-regularization curves:
  (a) Expected MSE as a function of the truncation parameter $r$. Plots are indexed by different values of $\sigma/\sqrt{n}$, and (b) overall contour plot of expected MSE for $\sigma/\sqrt{n}$ vs. $r$. }
\end{figure}

Next, we consider a bandlimited alignment with two non-overlapping bands of length $b$, starting at indices $\ell_1+1$ and $\ell_2+1$. The nonzero entries of $\xis$ are generated as before, followed by unit norm normalization. We fix $\lambda$ and consider the $r$-regularization curve, that is, the plot of MSE versus truncation parameter $r$. The results are shown in Figure~\ref{fig:2a} for different values of $\sigma/\sqrt{n}$. Figure~\ref{fig:2b} shows the corresponding contour plot.
The plots show the double-descent and phase transition behavior of the $r$-regularization curves. The double-decent behavior is consistent with Theorem~\ref{prop:band:lim}\ref{band:lim:r:curve} that predicts the MSE goes up as $r$ varies between bands. Simulations corroborating Theorem~\ref{thm:poly:rate} are provided in Appendix~\ref{app:add:sims}. The code for these experiments is available at GitHub repository \href{https://github.com/aaamini/krr-target-align}{aaamini/krr-target-align}~\cite{Amini_Target_alignment_in_2022}.

\section{Conclusion}\label{sec:conclusion}
We presented an analysis of the TKRR in the light of recent advances that were made for the KRR. We showed the TKRR equivalence to a KRR on a (smaller) RKHS embedded in the original RKHS. We believe that this insight provides a different perspective on the TKRR and unifies it with the KRR development.  We derived an exact MSE decomposition for the TKRR via TA scores and eigenvalues of the kernel matrix. When TA scores exhibit polynomial decay we identified an \emph{over-aligned} regime, where the TKRR can outperform ordinary KRR. %
In case the TA scores are bandlimited, we demonstrated the non-monotonicity of the regularization curve with respect to the level of truncation and a correspondence of narrower TA spectra that are positioned towards lower indices with lower generalization error.

\textbf{Limitations and extensions.} Although we restricted our attention to the empirical norm $\norm{f}_n^2$ and a fixed design setting (where $\{x_i\}$ are deterministic), it is possible to extend these results to the population $L_2$-norm, $\norm{f}^2 = \int f(x) d\mu(x)$, under a random design model $x_i \sim \mu$. This can be done using results such as those in~\cite[Chapter 14]{wainwright2019high} showing that $\norm{f}^2$ is within a constant factor of $\norm{f}_n^2$ for large $n$, uniformly over $f$. Our results are also easily extended to the non-additive and heteroscedastic noise setting, where $(x,y)$ are drawn from a general distribution, $f^*(x) := \ex[y \mid x]$ and $w = y - f^*(x)$. The extension is possible if we assume $\var(w \,|\, x) \le \sigma^2$, in which case, Theorem~\ref{thm:exact:mse}, for example, holds as an upper bound on the MSE. 
However, extensions of our results to more general loss functions, such as logistic loss for classification, requires additional work. Whether the improved rate~\eqref{eq:TKRR:optimal:rate} for TKRR is minimax optimal for $\gamma > 1$, appears to be open, although we conjecture it to be the case.
Due to the theoretical nature of this contribution, no negative societal impact of our work is anticipated.

\section{Proofs}
\label{proofs}

Here, we only give the proof of Theorem~\ref{thm:poly:rate}. The remaining proofs are given in Appendix~\ref{app:remaining:proofs}.

\vspace{-1ex}
\begin{proof}[Proof of Theorem~\ref{thm:poly:rate}]
	To simplify the argument, we assume that $\lambda^{-1/\alpha}$ is an integer, without loss of generality. Recall that for a decreasing function $f$, we have $\sum_{i=L}^U f(i) \le \int_{L-1}^U f(x) dx$. Let us first consider the sum involved in the variance term. 
	We will use the following inequalities, $\frac12 \min(a,b) \le \frac{ab}{a+b} \le \min(a,b)$ which holds for $a,b \ge 0$. We have
	\begin{align*}
			I_1 := \sum_{i=1}^r \frac{\mu_i^2}{(\mu_i + \lambda)^2}  = 
			\frac1{\lambda^2}\sum_{i=1}^r \frac{\mu_i^2\lambda^2}{(\mu_i + \lambda)^2} \asymp
			\frac1{\lambda^2}\sum_{i=1}^r \min(\lambda^2,\mu_i^2).
	\end{align*}
	Since $\mu_i = i^{-\alpha}$, the minimum above is $\lambda^2$  for $i \in [1,\lambda^{-1/\alpha}]$, and $\mu_i^2$ for $i > \lambda^{-1/\alpha}$. If $r \le \lambda^{-1/\alpha}$, only the first case happens and the overall bound is $\frac{1}{\lambda^2} r \lambda^2 = r$. If $r > \lambda^{-1/\alpha} =: k$, we have 
	\begin{align*}
		I_1 \asymp \frac1{\lambda^2} \Bigl( k \lambda^2 + \sum_{i=k+1}^r i^{-2\alpha}\Bigr) = k + \frac1{\lambda^2} \int_k^{r} x^{-2\alpha} dx \lesssim k + \lambda^{-2} k^{-2\alpha+1} \lesssim \lambda^{-1/\alpha}
	\end{align*}
	To summarize, $I_1 \lesssim \min(r, \lambda^{-1/\alpha})$ and we note that the variance term in~\eqref{eq:mse:expr} is $(\sigma^2/n) I_1$.
	
	For the middle term in~\eqref{eq:mse:expr}, we have, assuming $r < n$,
	\begin{align*}
		\sum_{i=r+1}^n (\xis_i)^2 = \sum_{i=r+1}^n i^{-2\gamma \alpha - 1} \le \int_r^\infty x^{-2\gamma \alpha-1}dx \lesssim r^{-2\gamma \alpha}.
	\end{align*}
	Finally,  for the first term in~\eqref{eq:mse:expr}, we can write
	\begin{align*}
		I_2 := \sum_{i=1}^r \frac{\lambda^2}{(\mu_i + \lambda)^2} (\xis_i)^2  = 
			\sum_{i=1}^r \frac{\lambda^2 \mu_i^2}{(\mu_i + \lambda)^2} \Bigl(\frac{\xis_i}{\mu_i}\Bigr)^2  \asymp
			\sum_{i=1}^r  \min(\lambda^2,\mu_i^2)\Bigl(\frac{\xis_i}{\mu_i}\Bigr)^2.
	\end{align*}
	If $r \le \lambda^{-1/\alpha}$, we have
	$	I_2 \le \lambda^2 \sum_{i=1}^r \bigl(\frac{\xis_i}{\mu_i}\bigr)^2 \asymp \lambda^2 \sum_{i=1}^r i^{-2(\gamma -1)\alpha -1}.$
	Letting $\beta = 2(\gamma-1)\alpha$,  %
	\begin{align*}
		\frac{1}{\lambda^2} I_2 \le 1 + \sum_{i=2}^r i^{-\beta -1} \le 1 + \int_1^r x^{-\beta-1}dx = 1 + \frac{1}{\beta}(1 - r^{-\beta}).
	\end{align*} 
	If $\gamma \ge 1$, the sum above is $\lesssim 1$. If $\gamma < 1$, then $\beta < 0$ and the sum is $\lesssim r^{|\beta|}$. The two cases can be compactly written as $\lesssim \max(1,r^{-\beta})$.
	If $r > \lambda^{-1/\alpha} := k$, we have
	\begin{align*}
		I_2 \;\le\; \lambda^2 \sum_{i=1}^k \Bigl(\frac{\xis_i}{\mu_i}\Bigr)^2 + \sum_{i=k+1}^r (\xis_i)^2 \;\lesssim\; \lambda^2 \max(1,k^{-\beta}) + k^{-2\gamma \alpha}
	\end{align*}
	using the bounding techniques we have seen before. Plugging in $k = \lambda^{-1/\alpha}$, we obtain
	\begin{align*}
		I_2 \;\lesssim\;  \lambda^2 \max(1,\lambda^{2(\gamma-1)}) + \lambda^{2\gamma} \lesssim \max(\lambda^2 ,\lambda^{2\gamma})
	\end{align*}
	To summarize, for $r \le \lambda^{-1/\alpha}$, we have $I_2 \lesssim \lambda^2 \max(1,r^{-\beta})$ while for $r > \lambda^{-1/\alpha}$, we have $I_2 \lesssim  \lambda^2 \max(1 ,\lambda^{2(\gamma-1)})$. Note that these two expressions can be combined into one as follows:
	\begin{align*}
		I_2 \; \lesssim \lambda^2 \max(1 , \eta^{-\beta}),
	\end{align*}
	where $\eta = \min(r, \lambda^{-1/\alpha})$. Putting the pieces together gives the desired bound. The above upper bounds are sharp up to constants, since in each step, there is a corresponding matching lower bound. %
	
	\textbf{Proof of part~(a).} With the given choice for $\lambda$ and $r$, we have $\lambda^{-1/\alpha} \asymp (n /\sigma^2)^{\gamma/(2\gamma \alpha+1)} \ge (n/\sigma^2)^{1/(2\gamma \alpha+1)} \asymp r$ for $\gamma \ge 1$.  Thus $\eta = \min(r, \lambda^{-1/\alpha} ) \asymp r$, hence $\eta^{-2(\gamma-1)\alpha} \lesssim 1$ for $\gamma \ge 1$ since $r \ge 1$. It follows that the 
	\begin{align*}
		\ex \empnorm{\ft_{r,\lambda} - \fs}^2 
		\;\asymp \; \lambda^2  + r^{-2\gamma \alpha} + \frac{\sigma^2}{n} r.
	\end{align*}
	Plugging in the choices for $\lambda$ and $r$, each of the three terms is $\asymp (\sigma^2/n)^{2\gamma\alpha/(2\gamma\alpha+1)}$ as desired.
	
	\textbf{Proof of part~(b).} Here, we have $r = n$, so the middle term in~\eqref{eq:poly:rate:mse} vanishes. Assume first that $\gamma \ge 1$. We consider different regularization regimes.
	
	Case 1: %
	$\lambda^{-1/\alpha} \in [1,n]$ so that $\lambda \in [n^{-\alpha},1]$
	and $\eta = \min(n,\lambda^{-1/\alpha}) = \lambda^{-1/\alpha}$. We have $\eta^{-2(\gamma-1)\alpha} = \lambda^{2(\gamma -1)} \le 1$. It follows that
	\begin{align*}
		\ex \empnorm{\ft_{r,\lambda} - \fs}^2 
		\;\asymp \; \lambda^2 +  \frac{\sigma^2}{n} \lambda^{-1/\alpha}.
	\end{align*}
	The optimal choice of $\lambda$ is obtained by setting $\lambda^2 \asymp (\sigma^2/n) \lambda^{-1/\alpha}$, that is, $\lambda = (\sigma^2/n)^{\alpha/(2\alpha+1)}$ and MSE rate $\asymp (\sigma^2/n)^{2\alpha/(2\alpha+1)}$. Note that, assuming $n^{-2\alpha} \lesssim \sigma^2 \lesssim n$, this choice of $\lambda$ is within the assumed range $[n^{-\alpha},1]$ up to constants.
	
	Case 2: $\lambda^{-1/\alpha} < 1$. Then, $\eta = 1$ and the MSE $\asymp \lambda^2 + \sigma^2/n$. Since $\sigma^2\lesssim n$ by assumption and $\lambda > 1$, we have MSE $\asymp \lambda^2 > 1$ 
	which is always worse than the rate in Case 1. %
	
	Case 3: $\lambda^{-1/\alpha} > n$. Then, $\eta = n$ and MSE $\asymp \lambda^2 + \sigma^2 \asymp \sigma^2$ since $\sigma^2 \gtrsim n^{-2\alpha}$ by assumption and $\lambda^2 < n^{-2\alpha}$. This rate is clearly worst than Case~1. Putting the pieces together the optimal rate is achieved by Case~1.
	
	Next, consider the case $\gamma < 1$. We again consider the three regularization regimes:
	
	Case 1: %
	$\lambda^{-1/\alpha} \in [1,n]$ so that $\lambda \in [n^{-\alpha},1]$
	and $\eta = \min(n,\lambda^{-1/\alpha}) = \lambda^{-1/\alpha}$. We have $\eta^{-2(\gamma-1)\alpha} = \lambda^{2(\gamma -1)} > 1$. It follows from~\eqref{eq:poly:rate:mse} that
	\begin{align*}
		\ex \empnorm{\ft_{r,\lambda} - \fs}^2 
		\;\asymp \;  \lambda^{2\gamma}+  \frac{\sigma^2}{n} \lambda^{-1/\alpha}.
	\end{align*}
	The optimal choice of $\lambda$ is obtained by setting $\lambda^{2\gamma} \asymp (\sigma^2/n) \lambda^{-1/\alpha}$, that is, $\lambda = (\sigma^2/n)^{\alpha/(2\gamma \alpha+1)}$ and MSE rate $\asymp (\sigma^2/n)^{2\gamma\alpha/(2\gamma \alpha+1)}$. Note that, as long as  $n^{-2\gamma \alpha } \lesssim \sigma^2 \lesssim n$ (which holds by the assumption on $\sigma^2$ since $\gamma < 1$), this choice of $\lambda$ is within the assumed range $[n^{-\alpha},1]$ up to constants.
	
	Case 2: $\lambda^{-1/\alpha} < 1$. This is similar to Case~2 when $\gamma \ge 1$.
	
	Case 3: $\lambda^{-1/\alpha} > n$. Then, $\eta = n$ and $\eta^{-2(\gamma-1)\alpha} > 1$, hence MSE $\asymp \lambda^2 n^{-2(\gamma-1)\alpha} + \sigma^2 \gtrsim \sigma^2$ which is clearly worst than the rate in Case~1. Putting the pieces together the optimal rate is achieved by Case~1.
\end{proof}

\bibliographystyle{plain}
\bibliography{krr_refs, kernel_refs}

\appendix

\section{Verifying reproducing property}\label{app:verify:reproduce}
To see the reproducing property of the kernel $\Kert$, defined in~\eqref{eq:Kt:def}, for the space $\hilt$, note that for any $f = \sum_\ell \alpha_\ell \psi_\ell \in \hilt$, we have 
\begin{align*}
	\hiltip{f, \Kert(\cdot,y)} %
	=\sum_{\ell=1}^r\sum_{k=1}^r \alpha_\ell \mu_k \psi_k(y)  \hiltip{\psi_\ell, \psi_k} = \sum_{k=1}^r  \alpha_k \psi_k(y) = f(y).
\end{align*}

\section{Remaining proofs}\label{app:remaining:proofs}

\begin{proof}[Proof of Proposition~\ref{prop:TKRR:equivalence}]

	We write $y = (y_1,\dots,y_n)$ and $w = (w_1,\dots,w_n)$. Then, the model can be compactly written as $y = \sqrt n S_{\xv}(\fs) + w$.
	Let $\yt = y / \sqrt n$ and $\wt = w / \sqrt n$ so that $\yt = S_{\xv}(\fs) + \wt$.
	
	By the representer theorem, a general TKRR solution is $\ft_{r, \lambda} = \St_{\xv}^*(\omt)$ where 
	$\omt$ is a solution of 
	\begin{align}\label{eq:Kt:optim}
		\min_{\omega \,\in\, \reals^n} \; 
		\frac1{n} \vnorm{y- \sqrt n \Kt \omega}^2 + \lambda
		\omega^T \Kt \omega.
	\end{align}
	The first-order optimality condition gives $\Kt[(\Kt \omt - \yt) + \lambda  \omt ] = 0$. Let us write $K = U \Lambda U^T$ for the eigen-decomposition of the full kernel matrix, where $\Lambda = \diag(\mu_1,\dots,\mu_n)$ and $U$ has columns $u_1,\dots,u_n$. Let $U_1 = [u_1 \mid u_2 \mid \cdots \mid u_r] \in \reals^{n \times r}$, $U_2 = [u_{r+1} \mid \cdots \mid u_n] \in \reals^{n \times (n-r)}$ and $\Lambda_1 = \diag(\mu_1,\dots,\mu_r) \in \reals^{r \times r}$. Then, $\Kt = U_1 \Lambda_1 U_1^T$ and $\omt = U_1 \alpha + U_2 \beta$ for some vector $\alpha$ and $\beta$. Substituting into the first-order condition and noting $U_1^T U_1 = I_r$ and $U_1^T U_2 = 0$, we have
	\begin{align*}
		U_1 \Lambda_1  [ ( \Lambda_1 \alpha - U_1^T\yt) + \lambda \alpha] = 0
	\end{align*}
	Let $\xi_{(1)} =  U_1^T\yt$. Multiplying both sides of the above by $\Lambda_1^{-1} U_1^T$, we obtain
	$(\Lambda_1 \alpha - \xi_{(1)}) + \lambda \alpha = 0.$
	Letting $A_\lambda = \Lambda_1 + \lambda I_r$, we have $\alpha = A_\lambda^{-1} \xi_{(1)}$. Thus all the solutions $\omt$ of~\eqref{eq:Kt:optim} are of the form 
	\begin{align}
		\omt = U_1 A_\lambda^{-1} \xi_{(1)} + U_2 \beta
	\end{align}
	for an arbitrary $\beta \in \reals^{n \times (n-r)}$. 
	
	Recalling $\ft_{r, \lambda} = \St_{\xv}^*(\omt)$, we have
	$\ft_{r, \lambda}(x_i) = \frac1{\sqrt n} \sum_{j=1}^n \omt_j \Kert(x_i, x_j) = \sqrt n [\Kt \omt]_i$,	
	hence
	\begin{align}\label{eq:Sx:ft}
		S_{\xv}(\ft_{r, \lambda}) = \Kt \omt = ( U_1 \Lambda_1 U_1^T) (U_1 A_\lambda^{-1} \xi_{(1)} + U_2 \beta) = U_1 \Lambda_1 A_\lambda^{-1} \xi_{(1)}
	\end{align}
	showing that $S_{\xv}(\ft_{r, \lambda})$ is the same for all the TKRR solutions. This proves parts~(a) and~(b). Part (c) is a consequence of $\Kt = K$ when $r = n$. The uniqueness also follows from the above argument since there is no $U_2\beta$ term in this case.	
	\end{proof}

	\begin{proof}[Proof of Theorem~\ref{thm:exact:mse}]
		Using the previous notation and recalling that $\xi^* = U^T S_{\xv}(\fs)$, we have 
			$\xi = U^T \yt = \xi^* + \bm z$ and $\bm z := U^T \wt$.
		Writing $\xi = U^T \yt = (\xi_{(1)}, U_2^T \yt)$, we can rewrite~\eqref{eq:Sx:ft} as
	$S_{\xv}(\ft_{r, \lambda}) = U \Gaml \xi.$
		It follows that
		\begin{align*}
			\empnorm{\ft_{r,\lambda} - \fs}^2 &= \norm{S_{\xv}(\ft_{r, \lambda}) - S_{\xv}(\fs)}_2^2  \\
			&= \norm{U \Gaml \xi  - U \xi^*}_2^2 \\
			&= \norm{\Gaml \xi  - \xi^*}_2^2 = 
			\norm{(\Gaml - I_n) \xi^*  +  \Gaml \bm z}_2^2.
		\end{align*}
		Expanding and using $\ex[\bm z] = 0$, we get
		\begin{align*}
			\ex \empnorm{\ft_{r,\lambda} - \fs}^2 = 
			\norm{(I_n - \Gaml) \xi^*}_2^2 +  \tr\bigl(\Gaml^2 \ex [\bm z \bm z^T]\bigr).
		\end{align*}
		Noting that $ \ex [\bm z \bm z^T] = \cov(\bm z) = U^T \cov(\wt) U = \frac{\sigma^2}n U^T U = \frac{\sigma^2}n I_n$ gives
			\begin{align*}
			\ex \empnorm{\ft_{r,\lambda} - \fs}^2 = 
			\norm{(I_n - \Gaml) \xi^*}_2^2 +   \frac{\sigma^2}n \tr(\Gaml^2)
		\end{align*}
		which is the desired result. The expression~\eqref{eq:mse:expr:2} is obtained by writing $\sum_{i=r+1}^n (\xis_i)^2 = \norm{\xis}_2^2 - \sum_{i=1}^r (\xis_i)^2$ and noting that $\empnorm{f} = \norm{\xis}_2$.
	\end{proof}

\begin{proof}[Proof of Proposition~\ref{prop:band:lim}]
	The expression for $\mseb$ follows by taking the expectation of both sides of~\eqref{eq:mse:expr:2} and noting that $\ex (\xis_i)^2 = 1/b$ when nonzero and $\ex \empnorm{\fs}^2 = 1$.
	
	For part (a), the assertions about the intervals $[1,\ell]$ and $[\ell+b+1,n]$ are immediate from the expression, since in both cases on the estimation error (third term) contributes to the $\mseb$ when increasing $r$. For the assertions regarding the middle interval, note that since $i \mapsto \mu_i$ is decreasing, we have $1 + \frac{2\lambda}{\mu_i} \le \frac{\sigma^2}n b$ for $i < j^*$ and $1 + \frac{2\lambda}{\mu_i} > \frac{\sigma^2}n b$ for $i \ge j^*$. The latter inequality is equivalent to $\frac1b a_i(\lambda) > \frac{\sigma^2}{n} \mu_i^2$ showing that the overall contribution of the $i$th terms of the two sums to the $\mseb$ is negative for $i > j^*$ (and by a similar argument nonnegative for $i \le j^*$.)
	
	For part (b), we note that the variable term of the $\mseb$ for $\ell \in [0,r-b]$ is 
	\begin{align}\label{eq:mse:middle:term}
		\frac1b \sum_{i=\ell+1}^{\ell+b} \frac{- a_i(\lambda)}{(\mu_i + \lambda)^2}  = \frac1b \sum_{i=\ell+1}^{\ell+b} \frac{\lambda^2}{(\mu_i +\lambda)^2} - 1.
	\end{align}
	Since $i \mapsto \frac{\lambda^2}{(\mu_i +\lambda)^2}$ is an increasing function, summing it over an sliding window of length $b$ starting at $\ell+1$, produces larger values as $\ell$ increases. 
	
	For part (c), the variable term of $\mseb$ is again~\eqref{eq:mse:middle:term} for $b \in [1,r-\ell]$. The variable part is the average of the sequence $i \mapsto \frac{\lambda^2}{(\mu_i +\lambda)^2}$ over a window of length $b$. Increasing the window length then increases the average since the sequence is increasing. 
\end{proof}

\section{Additional simulations}\label{app:add:sims}

\paragraph{Rate of TKRR vs. KRR}

\begin{figure}[t]
	\centering
	\begin{subfigure}{0.49\textwidth}
		\includegraphics[width=.99\textwidth]{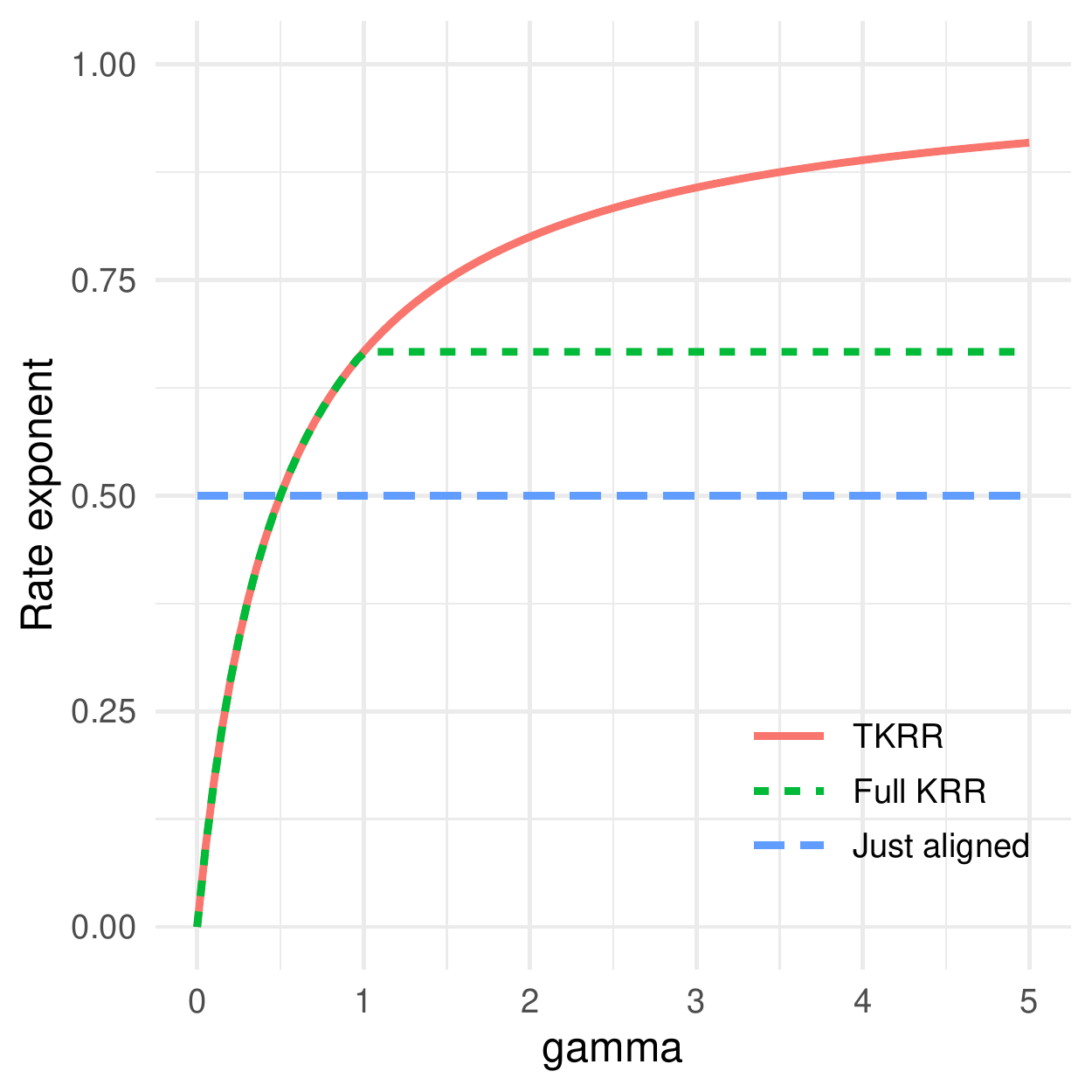} 
		\subcaption{Rate exponents}
	\end{subfigure}
	\begin{subfigure}{0.49\textwidth}
    	\includegraphics[width=.99\textwidth]{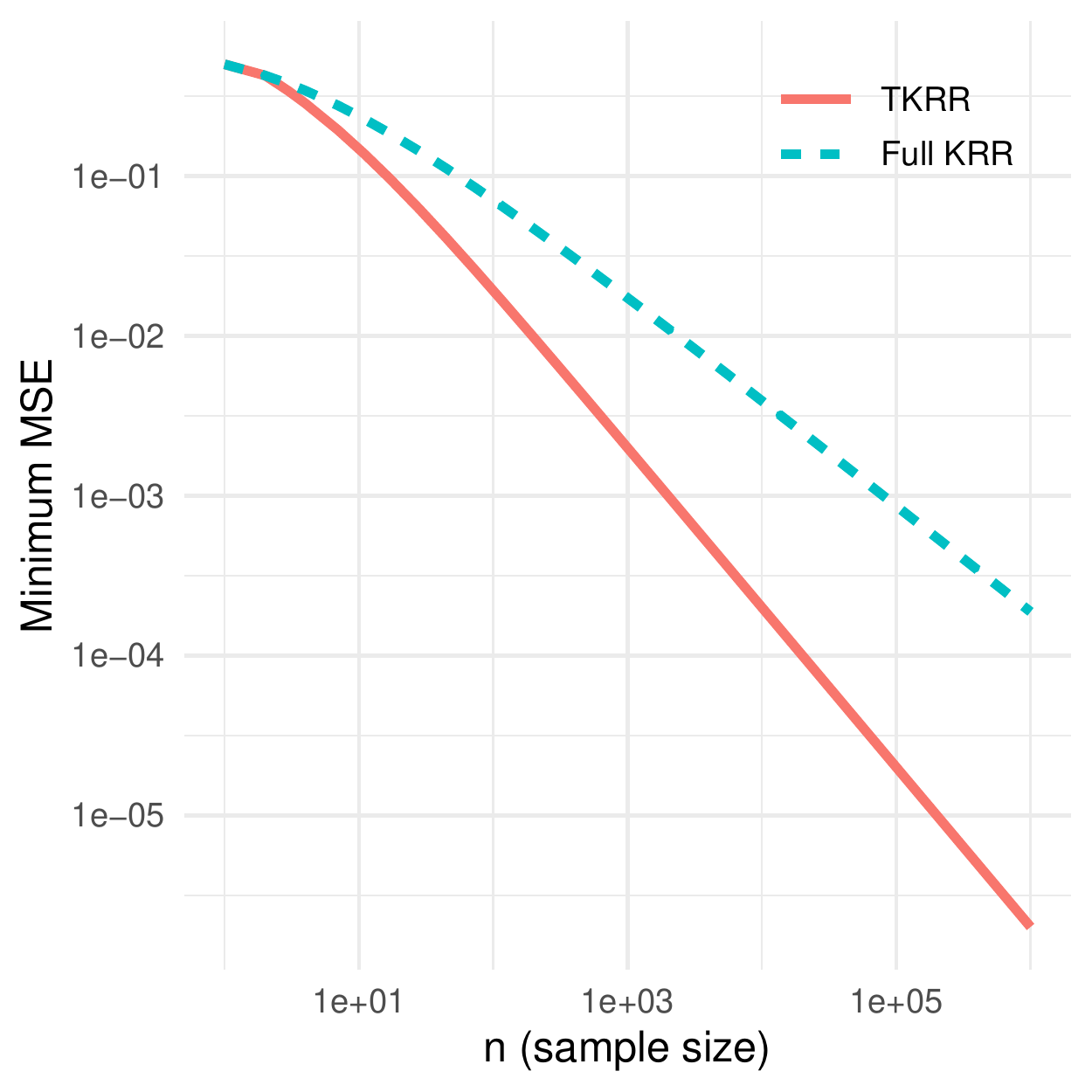}  
    	\subcaption{Minimum MSE vs. sample size}
	\end{subfigure}
	\caption{(a) The rate exponent function $s(\gamma)$ for the TKRR, Eq.~\eqref{eq:rate:exponent}, compared with that of full KRR $s(\delta) = s(1 \wedge \gamma)$ and the minimax exponent over the RKHS ball $s(1/2)$. (b) The minimum achievable MSE by TKRR and full KRR, as a function of the sample size, when $\alpha=1$ and $\gamma=10$.}
	\label{fig:basic:rate}
\end{figure}

We perform some experiments to coroborate the results of Theorem~ \ref{thm:poly:rate}. 
We let the eigenvalues and TA scores decay polynomially with %
rates specified as in~\eqref{eq:poly:decay:assump}, and take the truncation parameter $r$ to be as derived in Theorem~\ref{thm:poly:rate}\ref{part:tkrr:rate}.  For the full KRR and TKRR, we calculate the respective minimum value of MSE among 1000 values of the regularization parameter $\lambda$, evenly distributed between $10^{-10}$ and $10^2$. Figure~\ref{fig:basic:rate}(b) shows the MSE for the two methods, when $\alpha = 1$ and $\gamma = 10$, as a function of the sample size, on a log-log scale,. The difference in slope clearly shows the difference in rate between the two approaches.

The plots in Figure~\ref{fig:log:MSE:diff} show the difference in minimum log(MSE) between the full KRR and TKRR versus the sample size (on the log-scale), for different combinations of decay rate $\alpha$ and the noise level $\sigma$. For all the plots, we have $\gamma = 5$. According to Theorem~\ref{thm:poly:rate}, for sufficiently large $n$, the difference in minimum log(MSE) between the full KRR and TKRR should follow a line with positive slope when plotted as a function of $\log n$. This is clearly shown in  Figure~\ref{fig:log:MSE:diff}, where the positivity of the slope signifies the difference in rates between the two methods.

\begin{figure}[t]
\includegraphics[width=\textwidth]{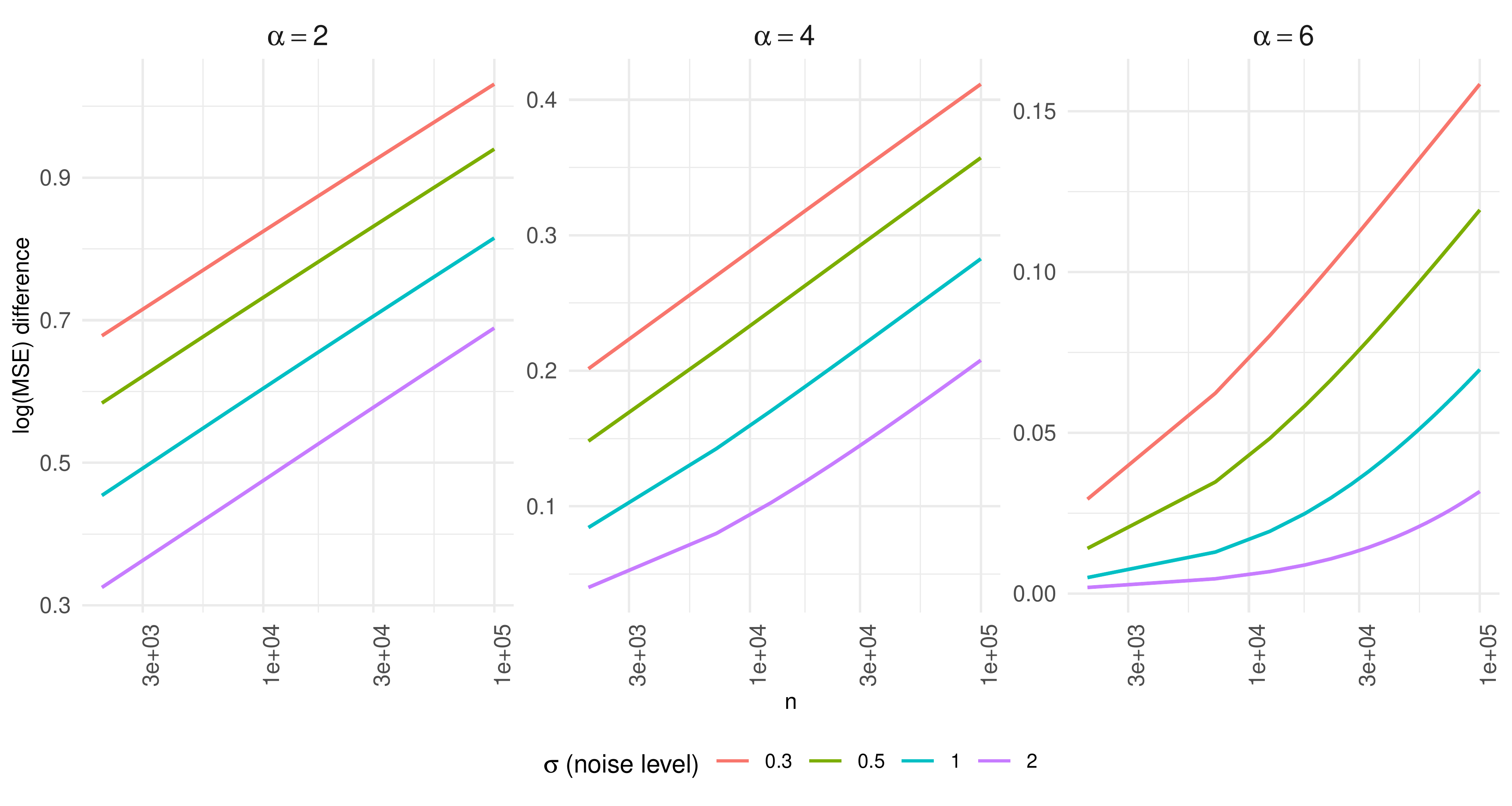}
\caption{The difference of log(MSE) between the full KRR and TKRR versus the sample size $n$ for $\gamma = 5$, and various values of $\alpha$ and the noise level $\sigma$. }
\label{fig:log:MSE:diff}
\end{figure}

\section{RKHS background}\label{app:rkhs:background}
 Assume that $\Xc$ is a measurable space with a $\sigma$-finite measure $\mu$ and $\hil$ is a separable RKHS over $\Xc$ with a measurable kernel $\Ker :\Xc \times \Xc \to \reals$. We write $L^2 := L^2(\mu)$ for the $L^2$ space of functions from $\Xc$ to $\reals$. For simplicity, we write $\ltwonorm{\Ker}$ for the $L^2$ norm of the function $x\mapsto \sqrt{\Ker(x,x)}$. We assume 
\begin{align}\label{eq:ker:l2:assump}
	\ltwonorm{\Ker} < \infty.
\end{align}
Then $\hil$ is a subset of $L^2$ and the inclusion map $J : \hil \to L^2$ is continuous. The adjoint of this map $J^* :L^2 \to \hil$ is the following integral operator
\begin{align*}
	J^* f(x) = \int \Ker(\cdot,x) f d\mu = \ltwoip{\Ker(\cdot,x), f} \quad f \in L^2.
\end{align*}
Let $T = JJ^* : L^2 \to L^2$. This can be thought of a the same integral operator acting on $L^2$ with output in $L^2$. The decomposition $T =JJ^*$ shows that $T$ is self-adjoint and positive. Condition~\eqref{eq:ker:l2:assump} implies that $T$ is a Hilbert-Schmidt, and hence a compact, operator. 

The spectral theorem for self-adjoint compact operators on $L^2$ implies that 
\[
	T f = \sum_{i \in I} \lambda_i e_i \ltwoip{f, e_i} \quad \text{ for all $f\in L^2$}
\]
where $\{\lambda_i\}_{i \in I}$ are the non-zero eigenvalues of $T$ order in decreasing fashion and $\{e_i\}_{i \in I} \subset L^2$ a corresponding sequence of eigenvectors (at most countable), forming an orthonormal system (ONS) in $L^2$. That is, $T e_i = \lambda_i e_i$ and $\ltwoip{e_i, e_j} = 1\{i = j\}$.

One can also view $\{e_i\}$ as functions in $\hil$, and it is not hard to see that $\{e_i\}_{i \in I}$ is an orthogonal sequence in $\hil$ with $\hilnorm{e_i}^2 = 1/{\lambda_i}$. That is, $\hip{e_i, e_j} = 1\{i = j\}/{\lambda_i}$. In other words, $\{\sqrt{\lambda_i} e_i\}_{i \in I}$ is an ONS in $\hil$.

Assume from now on that we are dealing with a Mercer kernel $\Ker$, that is, $\Xc$ is a compact space and $\Ker$ is a continuous function.
Then, we have the Mercer decomposition of the kernel function
\begin{align}
	K(x,y) = \sum_i \lambda_i e_i(x) e_i(y), \quad \forall x, y \in \Xc
\end{align}
where the convergence is uniform and absolute. It then follows that $\{\sqrt{\lambda_i} e_i\}$ is an orthonormal basis (ONB) of $\hil$ and we have
\begin{align*}
	\hil = \Bigl\{\sum_i \alpha_i e_i \mid \sum_i \frac{\alpha_i^2}{\lambda_i} < \infty\Bigr\}.
\end{align*}
The treatment up to this point follows more or less the treatment in~\cite[Chapter~4]{steinwart2008support}.

From now on, we patch the sequence $\{e_i\}_{i \in I}$ to a complete orthonormal basis for the entire $L^2$ namely $\{e_i\}_{i \in I'}$ where $I'$ is a proper subset of $I$. Let $I_0 := I'\setminus I$. Then, $e_i, i \in I_0$ span the orthogonal complement of the image of $T$ (i.e., the null space of $T$). We let $\lambda_i = 0$ for $i \in I_0$, so that 
\[
T f = \sum_{i \in I'} \lambda_i e_i \ltwoip{f, e_i} \quad \text{ for all $f\in L^2$}
\]
still holds. The statement $\hilnorm{e_i}^2 = \frac1{\lambda_i}$ also hold over $i \in I'$, interpreting $1/0$ as $\infty$. That is, $\hilnorm{e_i} = \infty$ when $i \in I_0$, consistent with the fact that such $e_i$ are not in $\hil$ (or more precisely do not have a version that is in $\hil$).

\subsection{Target alignment}
With this notation, every function in $L^2$ has a decomposition of the form $f = \sum_{i \in I'} \alpha_i e_i$ where $\alpha_i = \ltwoip{f,e_i}$. Then, the RKHS $\hil$ consists of those $f$ for which
\begin{align}\label{eq:hil:norm:alpha}
	\hilnorm{f}^2 = \sum_{i \in I'} \frac{\alpha_i^2}{\lambda_i} < \infty.
\end{align}
One can think of either $\{\alpha_i\}_{i \in I}$ or $\{\alpha_i\}_{i \in I'}$ as the population level kernel alignment spectrum (that is, the population counterpart of Definition~\ref{defn:TA}). Note that if $\alpha_i$ is nonzero for any $i \in I_0$, then $\hilnorm{f} = \infty$ and that $f$ is not in $\hil$. Even if $\alpha_i = 0$ for all $i \in I_0$,  $\{\alpha_i\}_{i \in I}$ needs to decay as imposed in~\eqref{eq:hil:norm:alpha} for the function to belong to the RKHS. For example, a necessary condition is $\alpha_i = o(\sqrt{\lambda_i})$ for $i \in I_0$. In other words, belonging to the RKHS itself implies some amount of alignment between the target and the kernel (i.e. some level of decay for $\{\alpha_i\}$.)

To summarize, we can write 
\begin{align*}
	\hil = \Bigl\{ f \in L^2 \mid \l \sum_{i \in I'} \frac{\ltwoip{f,e_i}^2}{\lambda_i} < \infty \Bigr\}.
\end{align*}

Let us connect to the setup of~\cite{caponnetto2007optimal} and~\cite{cui2021generalization}. In short, these two papers impose the following condition 
\begin{align}\label{eq:c:decay}
	 f = \sum_{i\in I'} \alpha_i e_i, \quad  \sum_{i \in I'} \frac{\alpha_i^2}{\lambda_i^c} < \infty
\end{align}
for some $c \in [1,2]$. If $c = 1$ this just means that $f \in \hil$. If $c > 1$ it means that it is in a proper subset of $\hil$. The $c$ here is the same as the $c$ in \cite{caponnetto2007optimal} and we have $c = 2r$ for parameter $r$ used in~\cite{cui2021generalization}. In our notation in this paper, $c = 2 \gamma$. (Note that in our paper, $r$ is reserved to the spectral truncation level and is a different parameter.)

In addition  \cite{caponnetto2007optimal} assumes $\lambda_i \asymp i^{-b}$ which is the same as the condition in~\cite{cui2021generalization}, that is, $\lambda_i \asymp i^{-\alpha}$, for $\alpha = b$. Here, our notation matches that of~\cite{cui2021generalization}; see~\eqref{eq:poly:decay:assump} which is the empirical counterpart of $\lambda_i \asymp i^{-\alpha}$.  Also, \cite{caponnetto2007optimal}  consider the case where $\lambda_i$ drop to zero exactly after some point (finite RKHS) which they refer to as the case $b = \infty$.

\subsection{Details of matching the setups}\label{app:rate:details}
The conditions~\cite{caponnetto2007optimal, cui2021generalization} are not stated as cleanly as~\eqref{eq:c:decay}. Let us see how they can be reformulated in this equivalent fashion. The paper~\cite{caponnetto2007optimal} which seems to be the origin of this condition works in the abstract setting of vector-valued RKHSs. We adapt the notation to the scalar-valued RKHSs. They work with operator $K_x: \reals \to \hil$ whose adjoint $K_x^* : \hil \to \reals$ is given by $K^*_x f= f(x)$ for every $f \in \hil$. We then have $a K^*_x f = \hip{K_x a, f}$ for any $a \in \reals$ by the definition of an adjoint operator. Since $a K^*_x f = a f(x) = \hip{a \Ker(\cdot,x), f}$, it follows that
\begin{align*}
	K_x a = a \Ker(\cdot,x), \quad a \in \reals.
\end{align*}
Then, they define the operator $T_x := K_x K_x^*: \hil \to \hil$ and $T = \int_{\Xc} T_x d\mu(x)$. We have 
\begin{align*}
		\ltwoip{e_i, T_x e_j} = \ltwoip{e_i, K_x K_x^{*} e_j} &= \ltwoip{e_i, K_x e_j(x)}  \\
	&= \ltwoip{e_i, e_j(x) K(\cdot, x)} \\
	&= e_j(x)   \ltwoip{e_i, K(\cdot, x)} = e_j(x) \lambda_i e_i(x)
\end{align*}
where the last step is since $e_i$ is an eigenvector of the integral operator $f \mapsto (x \mapsto \ltwoip{f, K(\cdot, x)})$ from $L^2$ to $L^2$, and that the range of this operator is in fact in $\hil$ (so evaluations make sense). This also follows from the Mercer decomposition.
It then follows that
\begin{align*}
	\ltwoip{e_i, T e_j} = \int 	\ltwoip{e_i, T_x e_j} d\mu(x) =\lambda_i \int e_i(x) e_j(x) d \mu(x) =
	\lambda_i \ltwoip{e_i, e_j} = \lambda_i 1\{i = j\}.
\end{align*}
That is $T$ can be viewed as a diagonal matrix $T = \diag(\lambda_i, i \in I')$ in the basis $\{e_i\}_{i \in I'}$. 

The condition in~\cite{caponnetto2007optimal} is $f = T^{(c-1)/2} g$ where $g \in \hil$ (or more precisely $\hilnorm{g}^2 \le R$). Let us write $g = \sum_{i \in I'} \beta_i e_i$ and $f = \sum_{i \in I'} \alpha_i e_i$. Since $T^{(c-1)/2}$ is a diagonal matrix in this basis, we have $\alpha_i = \lambda_i^{(c-1)/2}\beta_i$ or equivalently $\beta_i = \lambda_i^{(1-c)/2} \alpha_i$.
Then, $g \in \hil$ iff $\sum_i \beta_i^2 / \lambda_i < \infty$ which is equivalent to
\begin{align*}
	\sum_i \frac1{\lambda_i} (\lambda_i^{(1-c)/2})^2 \alpha_i^2 < \infty \iff 
	\sum_i \frac{ \alpha_i^2}{\lambda_i^c} < \infty
\end{align*}
and this is the desired condition.

Now to see that the condition in~\cite{cui2021generalization} is the same with $2r = c$, note that they require $\hilnorm{\Sigma^{1/2-r} \theta^*} < \infty$ in their equation~(7) which is a typo and is meant to be $\norm{\Sigma^{1/2-r} \theta^*}_{\ell^2} <\infty$ in the $\ell^2$ sequence norm. Here $\Sigma = \diag(\lambda_i)$ in our notation.

 As for $\theta^* = (\theta_i^*)$ which is a sequence in $\ell^2$, it is defined by the expansion $f^* = \sum_i \theta_i^* \psi_i$ where $\psi_i = \sqrt{\lambda_i} e_i$ in our notation. Thus, if we let $f = \sum_i \alpha_i e_i$, then $\alpha_i = \sqrt{\lambda_i} \theta_i^*$. So the condition imposed in~\cite{cui2021generalization} is 
\begin{align*}
	\sum_i (\lambda_i^{1/2 - r} \theta^*_i)^2 < \infty
	\iff \sum_i (\lambda_i^{1/2 - r} \lambda_i^{-1/2} \alpha_i)^2 < \infty \iff 
	\sum_i \lambda_i^{-2r} \alpha_i^2 < \infty
\end{align*}
which is the desired condition with $c = 2r$.

Immediately after stating this condition in~\cite{cui2021generalization}, it is abandoned in favor of the  condition $\lambda_i^{-2r} \alpha_i^2 \asymp i^{-1}$ which gives, together with $\lambda_i \asymp i^{-b}$
\begin{align*}
    \alpha_i \asymp i^{-1/2}\lambda_i^{r} \asymp  i^{-\frac{1 + 2r b}{2}}.
\end{align*}
or equivalently $\theta_i^* \asymp  \lambda_i^{-1/2} \alpha_i \asymp  O(i^{-\frac{1 + b(2r-1)}{2}})$. This is condition~(8) in~\cite{cui2021generalization}.

\subsection{Minimax rates}
Theorem~1 and~2 in~\cite{caponnetto2007optimal} together establish that the minimax rate for the signal model~\eqref{eq:c:decay} when $c \in (1,2]$ is given by $(1/\ell)^{bc/(bc + 1)}$, where $\ell$ is the sample size. Moreover, the same rate is minimax for $c = 1$ up to logarithmic factors. Translating to our notation with $c = 2\gamma$, $\ell = n$ and $b = \alpha$, the minimax rate in our model is $(1/n)^{2\gamma\alpha/(2\gamma \alpha + 1)}$ when $\gamma \in (1/2,1]$, as claimed in Section~\ref{subsec:polyAli}.

\end{document}